\documentclass[letterpaper, 10 pt, conference]{ieeeconf}  

\IEEEoverridecommandlockouts                              

\usepackage{amsmath,amssymb,amsfonts}
\usepackage{graphicx}
\usepackage[caption=false,font=footnotesize]{subfig} 
\usepackage{booktabs}
\usepackage{multirow}
\usepackage{array}
\usepackage{xcolor}
\usepackage{float}

\usepackage{cite}

\usepackage{algorithm}
\usepackage{algpseudocode}  

\let\proof\relax

\usepackage{amsthm}
\usepackage{mathtools}
\usepackage{mathrsfs}

\newtheorem{theorem}{\bf Theorem}[]

\newtheorem{lemma}{\bf Lemma}[]
\newtheorem{definition}{\bf Definition}[]

\title{\LARGE \bf
Observability-driven Assignment of Heterogeneous Sensors for Multi-Target Tracking
}

\author{Seyed Ali Rakhshan, Mehdi Golestani, and He Kong
\thanks{The authors’ work has been supported by the National Key R\&D Program of China under Grant No. 2024YFB4710902, the National Natural Science Foundation of China (NSFC) under Grant No. U24A20265, the Shenzhen Science and Technology Program under Grant No. KQTD20221101093557010, the Guangdong Science and Technology Program under Grant No. 2024B1212010002.}
\thanks{The authors are with the Guangdong Provincial Key Laboratory of Fully Actuated System Control Theory and Technology, School of Automation and Intelligent Manufacturing, Southern University of Science and Technology (SUSTech), Shenzhen 518055, China. Emails: rakhshan@sustech.edu.cn, golestani@sustech.edu.cn, kongh@sustech.edu.cn}%
}%

\begin{document}

\maketitle
\thispagestyle{empty}
\pagestyle{empty}

\begin{abstract}
This paper addresses the challenge of assigning heterogeneous sensors (i.e., robots with varying sensing capabilities) for multi-target tracking. We classify robots into two categories: (1) sufficient sensing robots, equipped with range and bearing sensors, capable of independently tracking targets, and (2) limited sensing robots, which are equipped with only range or bearing sensors and need to at least form a pair to collaboratively track a target. Our objective is to optimize tracking quality by minimizing uncertainty in target state estimation through efficient robot-to-target assignment. By leveraging matroid theory, we propose a greedy assignment algorithm that dynamically allocates robots to targets to maximize tracking quality. The algorithm guarantees constant-factor approximation bounds of \( 1/3 \) for arbitrary tracking quality functions and \( 1/2 \) for submodular functions, while maintaining polynomial-time complexity. Extensive simulations demonstrate the algorithm’s effectiveness in accurately estimating and tracking targets over extended periods. Furthermore, numerical results confirm that the algorithm's performance is close to that of the optimal assignment, highlighting its robustness and practical applicability.  
\end{abstract}

\section{INTRODUCTION}
Multi-robot systems enable a wide range of applications, including surveillance, search and rescue, and environmental monitoring \cite{r3,ISJ,r4,wang2024, wakulicz2021}. Their collective capabilities allow them to tackle complex, dangerous, or time-sensitive tasks effectively \cite{r12,su2021}. Among the key challenges in these systems is multi-target tracking, which requires robots to collaboratively estimate and follow dynamic targets' motion in real-time \cite{r15}. Crucial aspects of multi-target tracking include sensor placement and sensor-to-target assignment which have been discussed extensively in the existing literature \cite{TDOA, IOT, r8,m1,m2,m3}.

However, most existing works focus on the assignment problem with homogeneous sensors. Compared to homogeneous robot teams, heterogeneous teams that integrate robots with complementary sensing, mobility, and computational capabilities
can potentially enhance the entire system's performance across diverse environments \cite{r5,r9,r10,fu2024}. Nevertheless, these systems introduce challenges related to information exchange, decision-making, and task allocation, necessitating sophisticated coordination mechanisms.

\begin{figure}[h]
\centering
\includegraphics[width=0.48\textwidth]{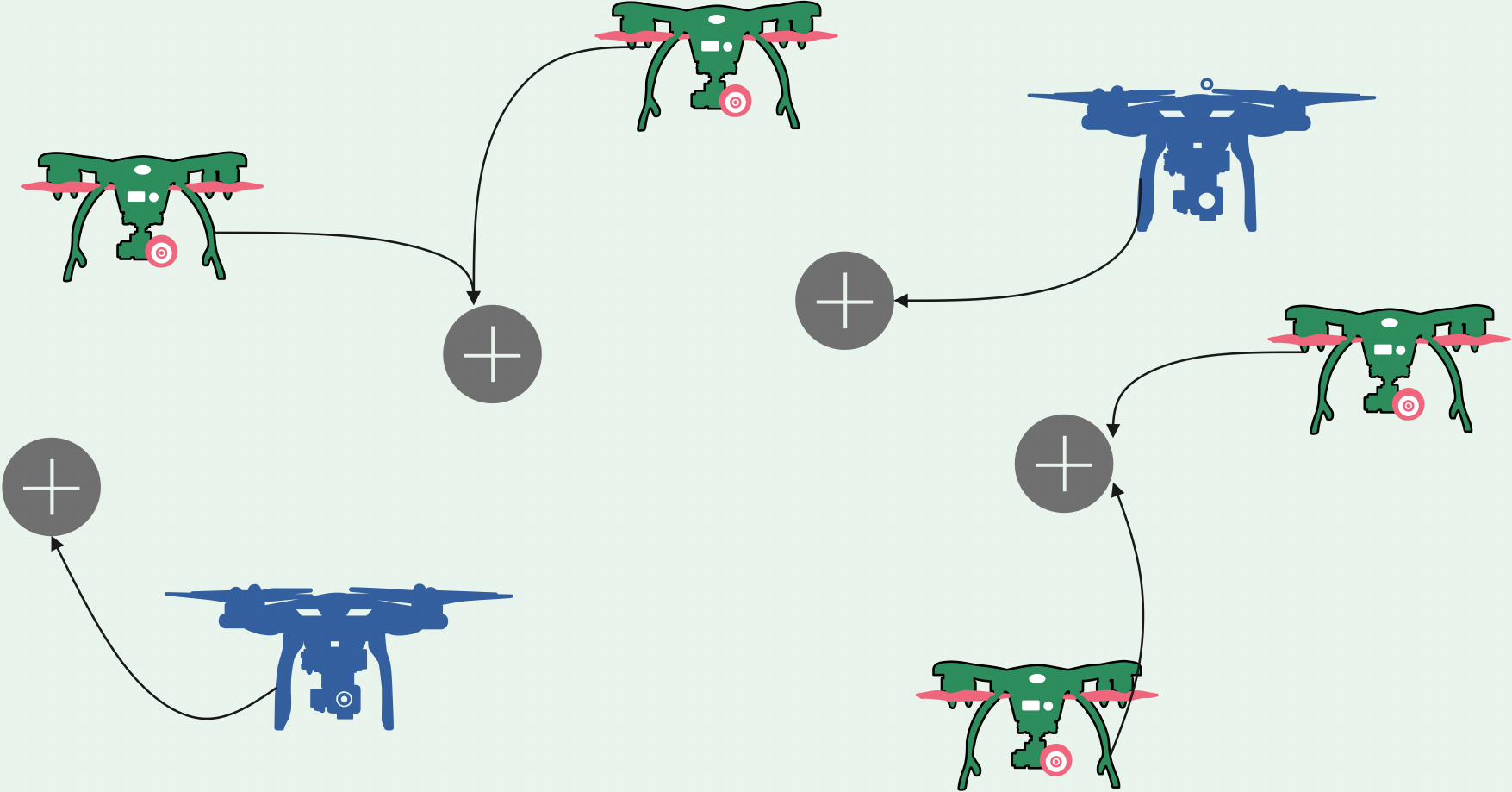}
\caption{Blue robots (equipped with range and bearing sensors) track targets (black circles) individually. Green robots (equipped with range-only sensors) require at least two to track a target collectively.}
\label{fig1}
\end{figure}

As shown in Fig.~\ref{fig1}, this paper addresses the problem of assigning heterogeneous sensors (robots with varying sensing capabilities) for multi-target tracking. To the best of our knowledge, this problem has not been previously studied. We consider two robot types: (1) sufficient sensing robots (e.g., with range and bearing sensors) that can track targets alone, and (2) limited sensing robots (with only range or bearing sensors) that must pair up for tracking. The goal is to minimize target state uncertainty by efficiently assigning robots to targets.

As a popular tool, matroid theory is particularly well-suited for handling complex combinatorial optimization problems, including those involving sensor networks and resource allocation \cite{n1,n2,n3,n4}. For the above-mentioned heterogeneous sensor assignment question, matroid theory can capture the combinatorial structure of the problem, thereby providing a natural way to model feasible solutions and allowing for efficient optimization. 

To be more specific, the feasible set of robot-target assignments can be modelled as a matroid, where the independent sets represent feasible assignments, and our goal is to find an optimal assignment that maximizes tracking quality while adhering to the robots' capabilities and sensor limitations. Via matriod theory, in this paper, we propose a greedy algorithm with constant-factor approximation bounds to efficiently allocate heterogeneous sensing robots to targets. The greedy algorithm works by iteratively selecting the most beneficial assignments, ensuring that the total quality of tracking is maximized without violating the system constraints. 

Our contributions are stated as follows. For the heterogeneous sensor assignment problem, building upon matroid theory, we introduce a greedy algorithm for efficient robot-target assignment. We rigorously establish that the proposed algorithm provides a \(1/3\)-approximation guarantee for arbitrary tracking quality functions and a \(1/2\)-approximation for non-decreasing submodular quality functions, both with polynomial-time complexity. We have also conducted extensive simulations to validate the algorithm’s effectiveness in accurately tracking targets over extended periods, with qualitative and quantitative performance assessment through empirical evaluation. Our results are similar to existing works, such as \cite{r1}, which established similar approximation bounds for the homogeneous sensor assignment problem. In this paper, we have modeled the heterogeneous assignment constraints as a matroid and leveraged its independence structure to ensure feasible assignments while maximizing observability. This theoretical advancement 
generalizes prior results to accommodate heterogeneous sensing platforms, providing a useful framework for optimal sensor allocation in complex multi-target tracking scenarios.

\section{RELATED WORKS}
A key challenge in multi-robot systems is multi-target tracking, which requires robots to collaboratively estimate and track dynamic targets in real-time. While earlier works, such as bipartite graph matching, effectively address one-to-one assignments \cite{r1,r2}, many-to-one assignments where multiple sensors track a single target are more complex due to diminishing returns. This complexity necessitates advanced algorithms for scalability and optimality, particularly in heterogeneous multi-robot systems where robots have complementary sensing, mobility, and computational capabilities.

The sensor-to-target assignment problem with homogeneous sensors has been widely studied, especially for multi-target tracking. For instance, \cite{r1} proposes approximation algorithms to improve estimator observability using greedy methods for monotone and submodular functions. Extensions to heterogeneous teams, such as in \cite{r8}, consider robots with sufficient or limited sensing.

Heterogeneous robots introduce coordination challenges but offer benefits like improved robustness and reduced cost. Optimizing estimation performance in such systems often involves formation control using Fisher information metrics \cite{1,2}. Recent efforts, such as \cite{L1}, address these challenges via architectures like scout–task decomposition with decentralized Monte Carlo tree search.



In the recent literature, resilient frameworks have been developed to allocate heterogeneous robots to tasks under adverse conditions, such as weather events or adversarial attacks (see, e.g., \cite{L2} and the references therein). 
Additionally, for scenarios with resource failures, \cite{L3} presents a resilient framework for networked heterogeneous multi-robot teams via communication network reconfiguration and one-hop observability. 


In a parallel line of research, sensor coverage, a fundamental problem in multi-robot systems, has also been addressed using heterogeneous robot teams. For example, \cite{L4} formulates sensor coverage as a graph representation learning problem, capturing heterogeneous relationships among robots to assign teams effectively. Their approach, based on regularized optimization, demonstrates the ability to learn unified representations of multi-robot systems and assign teams for optimal sensor coverage. This work emphasizes the importance of considering heterogeneous relationships in task assignment \cite{L5}. 

Similarly, task allocation in dynamic environments, such as emergency rescue scenarios, has been explored in \cite{L6}, which introduces a preference-driven approach based on hedonic games. By considering various preferences between robots and tasks, their distributed framework efficiently forms coalitions and converges to Nash-stable partitions even under strict communication limitations.

To address the complexity of multi-robot task allocation, hierarchical planning has emerged as a powerful tool. \cite{L7} presents a hierarchical planner that decomposes the problem into high-level robot allocation and low-level routing tasks. By using a Graph Neural Network as a heuristic to estimate subteam performance, their approach achieves near-optimal solutions with significantly reduced computation time. This demonstrates the potential of learning-based methods to improve the efficiency of task allocation in heterogeneous multi-robot systems. 

Distributed optimal control approaches have also been developed for managing omnidirectional sensor networks that cooperatively track moving targets \cite{m3}. The work in \cite{3} indicates that the performance of mobile sensors is significantly influenced by their distribution, with track coverage being an integral function of the sensor network's state. The introduction of normalized unused sensing capacity allows for the measurement of information gathered by each sensor relative to its theoretical maximum, facilitating a distributed coverage control strategy that adapts based on current usage \cite{4,5}.

Despite these advances, several gaps remain to be filled. First, most existing works on sensor-to-target assignment focus on homogeneous sensors (e.g., \cite{r8,r1}). Second, while some works address resilience and resource reallocation under adverse conditions (e.g., \cite{L2,L3}), they often do not consider the combinatorial structure of robot-target assignments. 


\section{PROBLEM FORMULATION}  


\subsection{Problem Setup}

\subsubsection{Robots and Targets}  
A team of \(N\) mobile robots, denoted by \(\mathcal{R} = \{1, \dots, N\}\), tracks \(M\) moving targets, denoted by \(\mathcal{T} = \{1, \dots, M\}\). The motion models for the robots and targets are described as follows:
\begin{align}\label{eq2}
\dot{\mathbf{x}}_{i,t} &= f_{i}(\mathbf{x}_{i,t}, \mathbf{u}_{i,t}), \quad \forall i \in \mathcal{R}, \\
\dot{\mathbf{y}}_{j,t} &= g_{j}(\mathbf{y}_{j,t}) + \mathbf{w}_{j,t}, \quad \forall j \in \mathcal{T},
\end{align}
where \(\mathbf{x}_{i}\) and \(\mathbf{y}_{j}\) represent the states of the \(i\)-th robot and the \(j\)-th target, respectively. The term \(\mathbf{u}_{i,t}\) denotes the control input of the \(i\)-th robot at the  \(t\)-th sampling instant. Additionally, the motion of the \(j\)-th target is stochastic, driven by zero-mean white Gaussian process noise with covariance \(\mathbf{Q}_{j,t}\), i.e., \(\mathbf{w}_{j,t} \sim \mathcal{N}(0, \mathbf{Q}_{j,t})\).

\subsubsection{Sensing Model}  

Each robot \(i\), observing target \(j\), receives measurements as follows: 
\begin{align}\label{eq3}
\mathbf{z}_{i,t}^{j} = h_{i}^{j}(\mathbf{x}_{i,t}, \mathbf{y}_{j,t}) + \mathbf{v}_{i,t}^{j},
\end{align}  
where \(\mathbf{z}_{i,t}^{j}\) denotes the measurement of target \(j\) by robot \(i\)’s sensor at time \(t\). The term \(\mathbf{v}_{i,t}^{j}\) represents Gaussian measurement noise with zero mean and covariance \(\mathbf{R}_{i,t}^{j}\), i.e., \(\mathbf{v}_{i,t}^{j} \sim \mathcal{N}(0, \mathbf{R}_{i,t}^{j})\).  

\subsubsection{Objective}  

The state of each target is estimated using an extended Kalman filter (EKF) based on measurements taken by the robots. Each target \(j\) is assigned a set of robot-actions, denoted by \(\psi(j)\). The tracking quality, denoted by \(q(\psi(j), j)\), can be computed using submodular and non-decreasing metrics, such as the log-determinant of the observability matrix. The goal is to assign heterogeneous robots to targets in order to maximize tracking quality.
\subsection{Problem Definition}  
The system consists of a robot set defined as:  
\begin{align}
\mathcal{R} = \mathcal{R}_s \cup \mathcal{R}_l = \{1, \dots, N_1\} \cup \{1, \dots, N_2\},
\end{align}
where \(\mathcal{R}_s\) represents the set of sufficient robots equipped with range and bearing sensors, and \(\mathcal{R}_l\) represents the set of limited sensing robots equipped with range or bearing only sensors. Given \(\mathcal{T} = \{1, \ldots, M\}\) as the set of targets, we require that each target \(j \in \mathcal{T}\) be tracked by either a single sufficient  robot \(i \in \mathcal{R}_s\), or a pair of limited robots \(\{i_1, i_2\} \subseteq \mathcal{R}_l\), with each robot assigned to at most one target. Therefore, one has \(|\mathcal{R}_s| + \left\lfloor \frac{|\mathcal{R}_l|}{2} \right\rfloor \geq M.\) The movements of the robots are determined via an optimal control problem. The core objective is to assign robots to targets in order to maximize tracking quality. Notably, a range-and-bearing sensor can measure both the distance and orientation of a target, making it sufficient to estimate the target's state \cite{r8}. In contrast, robots with limited sensing capabilities are not able to estimate the target's state individually. For example, it is well-known that with only range or bearing sensing, at least two robots are required to track a target \cite{r1}.

\subsubsection{Objective Function}

The objective is to maximize total tracking quality by assigning sufficient -sensing robots independently and limited-sensing robots in pairs:
\begin{align}
\max &\sum_{j \in \mathcal{T}} \bigg( \sum_{i \in \mathcal{R}_s} q(\psi(j), j) \cdot \gamma_{i,j} \cr
& + \sum_{i_1, i_2 \in \mathcal{R}_l, i_1 < i_2} q(\{\psi_1(j), \psi_2(j)\}, j) \cdot \zeta_{i_1,i_2,j} \bigg),
\end{align}
where \( \gamma_{i,j} \) is a binary assignment variable for sufficient -sensing robots, and \( \zeta_{i_1,i_2,j} \in \{-\infty,0,1\} \) represents pairwise assignments for limited-sensing robots. The values of \( \zeta_{i_1,i_2,j} \) are defined as:
    \begin{itemize}
        \item \(\zeta_{i_1,i_2,j} = 1\) means both robots \(i_1\) and \(i_2\) are assigned together to target \(j\),
        \item \(\zeta_{i_1,i_2,j} = 0\) means neither robot \(i_1\) nor \(i_2\) is assigned to target \(j\),
        \item \(\zeta_{i_1,i_2,j} = -\infty \) means only one of the two robots is assigned to target \(j\).
    \end{itemize}

\subsubsection{Constraints}

Each target is tracked by either one sufficient -sensing robot or a pair of limited-sensing robots:

\begin{align}
0 \leq \sum_{i \in \mathcal{R}_s} \gamma_{i,j} + \sum_{i_1, i_2 \in \mathcal{R}_l, i_1 < i_2} \zeta_{i_1,i_2,j} \leq 1, \quad \forall j \in \mathcal{T}.
\end{align}

\subsubsection{Observability Matrix and Tracking Quality Metric}
To quantify tracking quality, we introduce the concept of the observability matrix for both limited sensing robots (\(\mathcal{R}_l\)) and sufficient  sensing robots (\(\mathcal{R}_s\)). Using the Lie derivative, we define a tracking quality metric \( q \) based on the logarithm of the determinant of the symmetric observability matrix \( O^\top O \). The Lie derivative of the measurement function \( h_{i}^{j} \) with respect to the target state \( \mathbf{y}_{j,t} \), is given by:
\begin{align}\label{eq4}
L_{g} h_{i}^{j} = \frac{\partial h_{i}^{j}}{\partial \mathbf{y}_{j,t}} \cdot g(\mathbf{y}_{j,t}).
\end{align}

The observability matrix \( O \) is:

\begin{align}\label{eq5}
O  =\frac{\partial }{\partial \mathbf{y}_{j,t}} \begin{bmatrix}
[h_{1}^{j}, \cdots, h_{p}^{j}]^\top  \\
[L_{g} h_{1}^{j}, \cdots, L_{g} h_{p}^{j}]^\top \\
[L_{g}^2 h_{1}^{j}, \cdots, L_{g}^2 h_{p}^{j}]^\top \\
\vdots \\
[L_{g}^{T-1} h_{1}^{j}, \cdots, L_{g}^{T-1} h_{p}^{j}]^\top
\end{bmatrix}.
\end{align}
where \(p \) is the number of robots that are used to generate observability matrix \(O \), and \(T\) is the order derivative. 

\section{ASSIGNMENT ALGORITHM AND ANALYSIS}
In this section, we present our proposed greedy algorithm, i.e., Algorithm \ref{al1} where $q(\text{GREEDY})$ denotes total value obtained by the greedy approach, and show that it possess constant factor approximation bounds while running in polynomial time. 
\begin{algorithm} 
\caption{Greedy Assignment for Heterogeneous Sensor Assignment}\label{al1}
\begin{algorithmic}[1]
\State $k \gets 0$, $q(\text{GREEDY}) \gets 0$
\While{true}
    \State Compute all possible $q(\psi_i(j), j)$ when $ i \in \mathcal{R}_s$ and all possible $q(\psi_{i_1}(j),\psi_{i_2}(j) , j)$ when $ i_1, i_2\in \mathcal{R}_l$.
    \State Select one sufficient -sensing robot or one pair of limited-sensing robots with maximum $q(\cdot , j)$ defined as $q_{\max}$.
    \State $q(\text{GREEDY}) \gets q(\text{GREEDY}) + q_{\max}$.
    \State Remove target $j$ from the target set $\mathcal{T}$ and selected robot/robots from the robot set $\mathcal{R}$.
    \State $k \gets k + 1$
\EndWhile
\end{algorithmic}
\end{algorithm}
\subsection{Assignment Algorithm with Arbitrary Tracking Quality Function}
Initially, we explore the scenario in which the tracking quality function is arbitrary and have the following results in Theorem \ref{th1} (the proof can be derived by applying the principles outlined in Theorems 1 and 2 from \cite{r8}, and is therefore skipped here). 

\begin{theorem}\label{th1}
Let the environment consist of both sufficient  sensing robots and limited sensing robots, i.e., \(\mathcal{R} = \mathcal{R}_s \cup \mathcal{R}_l\). In this case, the performance of the GREEDY algorithm satisfies the inequality \(q(\text{GREEDY}) \geq \frac{1}{3} q(\text{OPT})\), where \(\text{OPT}\) represents the optimal algorithm. Additionally, the running time is \(O((N_1+N_2^2) M^2)\).
\end{theorem}
\subsection{Assignment Algorithm with Non-decreasing Submodular Tracking Quality Function}
For the assignment problem of heterogeneous sensing robots for multi-target tracking, matroid theory provides a viable framework for defining feasible robot sets and enabling efficient optimization. 

\begin{definition}  
A matroid is an ordered pair \( \mathcal{M} = (S, \mathcal{I}) \) satisfying the following conditions \cite[Chap. 16]{new81}: 
\begin{itemize}  
    \item A finite set \( S \), and  
    \item A collection \( \mathcal{I} \) of independent subsets of \( S \), satisfying:  
    \begin{enumerate}  
        \item \textbf{Non-empty set}: \( \emptyset \in \mathcal{I} \),  
        \item \textbf{Hereditary property}: If \( A \in \mathcal{I} \) and \( B \subseteq A \), then \( B \in \mathcal{I} \),  
        \item \textbf{Exchange property}: If \( A, B \in \mathcal{I} \) and \( |A| < |B| \), there exists \( x \in B \setminus A \) such that \( A \cup \{x\} \in \mathcal{I} \).  
    \end{enumerate}  
\end{itemize}  
\end{definition}  

\begin{definition}\label{m1}
The robot-target assignment matroid is defined as the transversal matroid \cite[Chap. 5]{new82} \(\mathcal{M} = (\mathcal{T}, \mathcal{I})\) over a bipartite graph \(G = (U, V, E)\), where:
\begin{itemize}
    \item \(U = \mathcal{R}_s \cup \left\{ \{i_1, i_2\} \mid i_1 \neq i_2 \in \mathcal{R}_l \right\}\) is the set of assignable units (individual sufficient  robots or unordered pairs of limited ones),
    \item \(V = \mathcal{T}\) is the set of targets,
    \item \(E \subseteq U \times V\) includes \((i, j)\) for all \(i \in \mathcal{R}_s\), and \((\{i_1, i_2\}, j)\) for all valid limited robot pairs and targets \(j \in \mathcal{T}\).
\end{itemize}
A subset \(S \subseteq \mathcal{T}\) is independent (\(S \in \mathcal{I}\)) if there exists a matching in \(G\) that covers \(S\), assigning to each \(j \in S\) either one sufficient  robot or one distinct pair of limited robots, such that no robot is reused. The rank of \(\mathcal{M}\) is the size of the largest such matching; the bases are its maximal matchable subsets.
\end{definition}
\begin{lemma}
The robot-target assignment matroid \(\mathcal{M} = (\mathcal{T}, \mathcal{I})\), as defined in Definition \ref{m1}, is a matroid. 
\end{lemma}
\proof Let \(\mathcal{M} = (\mathcal{T}, \mathcal{I})\) be the robot-target assignment matroid defined in Definition \ref{m1}. We need to verify each matroid condition.
\begin{enumerate}
    \item The empty set \(\emptyset \subseteq \mathcal{T}\) requires no targets to be covered by a matching. The empty matching (containing no edges) is a valid matching in \(G\), as it satisfies all constraints. Thus, \(\emptyset \in \mathcal{I}\).
    \item Suppose \(A \in \mathcal{I}\) and \(B \subseteq A\). Since \(A \in \mathcal{I}\), there exists a matching \(M_A\) in \(G\) that covers \(A\). For each target \(j \in A\), there is an edge \((u, j) \in M_A\), where \(u \in U\) (either \(u \in \mathcal{R}_s\) or \(u = \{i_1, i_2\} \in \left\{ \{i_1, i_2\} \mid i_1, i_2 \in \mathcal{R}_l, i_1 \neq i_2 \right\}\)), and no two edges in \(M_A\) share a vertex in \(U\). Moreover, the robots (sufficient  or limited) involved in \(M_A\) are distinct, ensuring each robot is used at most once. To cover \(B \subseteq A\), construct a matching \(M_B\) by taking the subset of edges in \(M_A\) incident to targets in \(B\). Since \(M_B \subseteq M_A\), it is a valid matching (using a subset of the robots), and it covers \(B\). Thus, \(B \in \mathcal{I}\). Hence, the hereditary property holds.

    \item Let \(A, B \in \mathcal{I}\) with \(|A| < |B|\). Since \(A, B \in \mathcal{I}\), there exist matchings \(M_A\) and \(M_B\) in \(G\) covering \(A\) and \(B\), respectively. Consider the symmetric difference graph \(G' = (U, V, M_A \triangle M_B)\), where \(M_A \triangle M_B = (M_A \setminus M_B) \cup (M_B \setminus M_A)\). In \(G'\), each vertex in \(\mathcal{T}\) has degree at most 2 (one edge from \(M_A\), one from \(M_B\)), and each vertex in \(U\) has degree at most 2 (since each \(u \in U\) is matched at most once in each matching). Thus, \(G'\) consists of disjoint paths and cycles, with edges alternating between \(M_A\) and \(M_B\).

    Since \(|A| < |B|\), there are more edges in \(M_B\) than in \(M_A\) incident to \(B \setminus A\). Choose a target \(x \in B \setminus A\). In \(G'\), \(x\) is incident to an edge \((u, x) \in M_B\). If \(x\) is not incident to any edge in \(M_A\), then \(u\) (an sufficient  robot or a pair of limited robots) is not used in \(M_A\). Adding \((u, x)\) to \(M_A\) forms a matching \(M_A \cup \{(u, x)\}\) that covers \(A \cup \{x\}\). Since the robots in \(u\) are distinct from those in \(M_A\), this matching respects the constraint that each robot is used at most once, so \(A \cup \{x\} \in \mathcal{I}\).

    If \(x\) lies on an alternating path in \(G'\), there must be an augmenting path from \(x\) to some target \(y \in A \) (since \(|B| > |A|\), the number of \(M_B\)-edges exceeds \(M_A\)-edges, ensuring such a path exists). Augmenting \(M_A\) along this path (taking edges in \(M_B \setminus M_A\) and omitting edges in \(M_A \setminus M_B\)) produces a new matching of size \(|A| + 1\) that covers \(A \cup \{x\}\). The robots used in this new matching remain distinct, as the path alternates between matchings and respects the robot assignment constraints. Thus, \(A \cup \{x\} \in \mathcal{I}\). Hence, the exchange property holds.
\end{enumerate}
Since \(\mathcal{M} = (\mathcal{T}, \mathcal{I})\) satisfies all three matroid conditions, it is a matroid. \hfill $\blacksquare$

Let \( U^{t} \) be the set of assignable units (sufficient  robots or pairs of limited robots) considered in the first \( t + 1 \) iterations of the greedy algorithm before the addition of the \( (t + 1) \)-st unit. Also, consider the robot-target assignment matroid \(\mathcal{M} = (\mathcal{T}, \mathcal{I})\) as defined in Definition~\ref{m1}, where \(\mathcal{T} = \{1, \ldots, M\}\) is the set of targets, and \(\mathcal{I}\) is the collection of independent sets corresponding to matchings in the bipartite graph \(G = (U, V, E)\), with \(U = \mathcal{R}_s \cup \left\{ \{i_1, i_2\} \mid i_1, i_2 \in \mathcal{R}_l, i_1 \neq i_2 \right\}\), \(V = \mathcal{T}\), and edges \(E\) connecting each target to either an sufficient  robot or a pair of limited robots. Define the rank function \( r(S) \) for \( S \subseteq \mathcal{T} \) as the cardinality of the largest independent set contained in \( S \), i.e., the size of the largest matching in \( G \) covering a subset of \( S \). Define the span of \( S \) in the matroid as
\begin{align}
\text{span}(S) = \{ j \in \mathcal{T} : r(S \cup \{ j \}) = r(S) \}.
\end{align}

\begin{theorem}
Consider the robot-target assignment matroid \(\mathcal{M} = (\mathcal{T}, \mathcal{I})\) as defined in Definition~\ref{m1}, where \(\mathcal{T}\) is the set of targets, and \(\mathcal{I}\) represents feasible assignments of targets to either an sufficient  robot from \(\mathcal{R}_s\) or a distinct pair of limited robots from \(\mathcal{R}_l\). If the greedy algorithm is applied to the robot-target assignment problem with matroid \(\mathcal{M}\), then
\begin{align}
q(\text{GREEDY}) \geq \frac{1}{2} q(\text{OPT}),
\end{align}
where \( q(\cdot) \) is a nondecreasing submodular tracking quality function, \( q(\text{OPT}) \) is the optimal value of the tracking quality over all feasible assignments, and \( q(\text{GREEDY}) \) is the value obtained from the greedy algorithm.
\end{theorem}

\begin{proof}
Let \( T \subseteq \mathcal{T} \) denote the optimal solution, i.e., a maximum independent set in \(\mathcal{I}\) that maximizes \( q(T) \), and let \( S \subseteq \mathcal{T} \) denote the greedy solution with \( |S| = M \), where \( M \) is the number of targets. Let \( S^{t} \subseteq \mathcal{T} \) denote the greedy solution at iteration \( t \) of the greedy algorithm, with \( S^{0} = \emptyset \). For each iteration \( t = 1, \dots, M \), and let \( u(t) \in U \) be the \( t \)-th assignable unit (either an sufficient  robot from \(\mathcal{R}_s\) or a pair of limited robots from \(\{ \{i_1, i_2\} \mid i_1, i_2 \in \mathcal{R}_l, i_1 \neq i_2 \}\)) selected by the greedy algorithm, and define
\begin{align}
\rho_{t} = \rho_{u(t)}(S^{t}) = q(S^{t} \cup \{ j(t) \}) - q(S^{t}),
\end{align}
where \( j(t) \in \mathcal{T} \) is the target assigned to unit \( u(t) \) at iteration \( t \). Define
\begin{align}
s_{t-1} = |T \cap (U^{t} - U^{t-1})|,
\end{align}
where \( U^{t-1} \subseteq U \) represents the set of all assignable units considered in the first \( t \) iterations before adding the \( t \)-th unit.

From Proposition 2.2 in \cite{r7}, we have
\begin{align}
\sum_{t=1}^{M} \rho_{t-1} s_{t-1} \leq \sum_{t=1}^{M} \rho_{t-1}.
\end{align}
Based on the submodularity of \( q \) and Proposition 2.1 of \cite{r6}, we derive
\begin{align}
q(\text{OPT}) &= q(T) \leq q(S) + \sum_{j \in T \setminus S} \rho_j(S) \leq q(S) + \sum_{j \in T} \rho_j(S) \notag \\
&= q(S) + \sum_{t=1}^{M} \sum_{j \in T \cap (U^{t} - U^{t-1})} \rho_j(S) \notag \\
&\leq q(S) + \sum_{t=1}^{M} \rho_{t-1} s_{t-1} \leq q(\text{GREEDY}) + \sum_{t=1}^{M} \rho_{t-1} \notag \\
&\leq q(\text{GREEDY}) + (q(\text{GREEDY}) - q(\emptyset)),
\end{align}
which implies
\begin{align}
q(\text{GREEDY}) \geq \frac{1}{2} q(\text{OPT}).
\end{align}
This completes the proof. \hfill $\blacksquare$
\end{proof}

\section{SIMULATION RESULTS AND DISCUSSIONS}
In this section, we evaluate the performance of the proposed greedy algorithm through comprehensive simulations. Note that the algorithms are independent of the motion models. The motion models introduced here are for simulation and validation purposes only. Moreover, the main assignment problem is NP-complete, rendering it computationally intractable to derive the optimal solution in polynomial time, especially for large-scale instances. Consequently, we restrict our comparison between the optimal assignment and the greedy solution to scenarios with a small number of targets, where the optimal solution can still be feasibly computed. Furthermore, the state of each target is estimated using an Extended Kalman Filter (EKF) based on measurements obtained from the robot sensors.

\subsection{Robot Motion Model}
Each robot \(i \in \mathcal{R}\) operates according to the unicycle motion model, described by the following equations:  

\begin{align}  
\begin{pmatrix}  
\dot{\mathbf{x}}_{i,t}^{1} \\
\dot{\mathbf{x}}_{i,t}^{2} \\
\dot{\theta}_{i,t}  
\end{pmatrix}  
=  
\begin{pmatrix}  
v_{i}  \cos(\theta_{i,t}) \\
v_{i}  \sin(\theta_{i,t}) \\
\omega_{i}  
\end{pmatrix},  
\end{align}  
where the state of the robot is represented as \(\mathbf{x}_{i} = [x_{i,t}^{1}, x_{i,t}^{2}, \theta_{i,t}]^{\top}\), with \([x_{i,t}^{1}, x_{i,t}^{2}]^{\top}\) indicating the position in the 2D plane and \(\theta_{i,t}\) representing the robot's orientation relative to the world frame. The vector \(\mathbf{u}_{i} = [v_{i}, \omega_{i}]^{\top}\) denotes the robot's action (or control input), where \(v_{i}\) and \(\omega_{i}\) are the linear and angular velocities, respectively.

\subsection{Target Motion Model}
Each target \( j \in \mathcal{T} \) follows a circular motion trajectory with added Gaussian noise. The motion dynamics are given by:

\begin{equation}  
\begin{pmatrix}  
    \dot{\mathbf{y}}_{j,t}^{1} \\  
   \dot{\mathbf{y}}_{j,t}^{2}  
\end{pmatrix}  
=  
\begin{pmatrix}  
    -d_{j} \phi_j \sin(\phi_{j}t) \\  
    d_{j} \phi_j \cos( \phi_{j}t)  
\end{pmatrix}  
+ \mathbf{w}_{j,t},  
\end{equation}
where the noise vector is given by \( \mathbf{w}_{j,t} \sim \mathcal{N}(0, \mathbf{Q}) \), with covariance matrix:

\begin{equation}  
\mathbf{Q} = \begin{bmatrix}  
    \sigma_{j}^{2} & 0 \\
    0 & \sigma_{j}^{2}  
\end{bmatrix}, \quad  \sigma_{j}=0.2.  
\end{equation}

\subsection{Observation Model}
The range sensor measurement model is
\begin{equation}  
h_r(\mathbf{x}_i, \mathbf{y}_j) = \frac{1}{2} \bigg( (y_{j}^{2} - x_{i}^{2})^{2} + (y_{j}^{1} - x_{i}^{1})^{2} \bigg)+ \mathbf{v}_{r_i}^{j},
\end{equation}
where \(\mathbf{v}_{r_i}^{j} \sim \mathcal{N}(0, \sigma_r^2) \), \( \sigma=0.2 \). Also, the bearing sensor measurement model is
\begin{equation}  
h_b(\mathbf{x}_i, \mathbf{y}_j) = \text{atan2}(y_{j}^{2} - x_{i}^{2}, y_{j}^{1} - x_{i}^{1}) - \theta_{i}+ v_{b_i}^{j},
\end{equation}
where \(\mathbf{v}_{b_i}^{j} \sim \mathcal{N}(0, \sigma_b^2) \), \( \sigma=0.2 \). For a single robot \( i \in \mathcal{R}_s \), the tracking quality is logarithm determinant of symmetric observability matrix constructed by the Lie derivative based on \eqref{eq4} and \eqref{eq5} as follows:
\begin{equation}
\frac{\partial h_r}{\partial \mathbf{y}_j} =
\begin{bmatrix} (y_j^1 - x_i^1) & (y_j^2 - x_i^2) 
\end{bmatrix}^T,
\end{equation}
\begin{equation}
L_g h_r = -(y_j^1 - x_i^1) d_{j} \phi_j \sin(\phi_j t) + (y_j^2 - x_i^2) d_{j} \phi_j \cos(\phi_j t).
\end{equation}
Also we have 
\begin{equation}
\frac{\partial h_b}{\partial \mathbf{y}_j} =
\frac{1}{(y_j^1 - x_i^1)^2 + (y_j^2 - x_i^2)^2}
\begin{bmatrix}-( y_j^2 - x_i^2 )\\  (y_j^1 - x_i^1) \end{bmatrix}^T,
\end{equation}
and,
\begin{equation}
L_g h_b = \frac{(y_j^2 - x_i^2) d_{j} \phi_j \sin(\phi_j t) + (y_j^1 - x_i^1) d_{j} \phi_j \cos(\phi_j t)}{(y_j^1 - x_i^1)^2 + (y_j^2 - x_i^2)^2}.
\end{equation}
Suppose
\begin{align}
&a=\frac{\partial L_g h_b }{\partial y_j^1} = \frac{d_{j} \phi_j \cos(\phi_j t) }{\left[(y_j^1 - x_i^1)^2 + (y_j^2 - x_i^2)^2\right]}\cr
& -\frac{2(y_j^1 - x_i^1)d_{j} \phi_j \bigg[(y_j^2 - x_i^2)\sin(\phi_j t)+(y_j^1 - x_i^1)\cos(\phi_j t) \bigg]}{\left[(y_j^1 - x_i^1)^2 + (y_j^2 - x_i^2)^2\right]^2},\cr
\end{align}
and
\begin{align}
&b=\frac{\partial L_g h_b }{\partial y_j^2} = \frac{d_{j} \phi_j \sin(\phi_j t) }{\left[(y_j^1 - x_i^1)^2 + (y_j^2 - x_i^2)^2\right]}\cr
& -\frac{2(y_j^2 - x_i^2)d_{j} \phi_j \bigg[(y_j^2 - x_i^2)\sin(\phi_j t)+(y_j^1 - x_i^1)\cos(\phi_j t) \bigg]}{\left[(y_j^1 - x_i^1)^2 + (y_j^2 - x_i^2)^2\right]^2},\cr
\end{align}
then, observability matrix is obtained based on Lie derivative as follows:
\begin{equation}
O=
\begin{bmatrix}

    (y_j^1 - x_i^1) & (y_j^2 - x_i^2) \\
    \frac{-(y_j^2 - x_i^2)}{(y_j^1 - x_i^1)^2 + (y_j^2 - x_i^2)^2} & \frac{(y_j^1 - x_i^1)}{(y_j^1 - x_i^1)^2 + (y_j^2 - x_i^2)^2} \\
    -d_j \phi_j \sin(\phi_j t) &  d_j \phi_j \cos(\phi_j t) \\
    a & b
\end{bmatrix}.
\end{equation}
Similarly, for a pair of robots \( i_1, i_2 \in \mathcal{R}_l \) observability matrix is obtained as follows:

\begin{equation}
O =
\begin{bmatrix}

    (y_j^1 - x_{i_1}^1) & (y_j^2 - x_{i_1}^2) \\
    (y_j^1 - x_{i_2}^1) & (y_j^2 - x_{i_2}^2) \\
    d_j \cos(\phi_j t) &  d_j \sin(\phi_j t) \\
\end{bmatrix}.
\end{equation}
\begin{figure}
\centering
\subfloat[\( k=50 \)]
{{\includegraphics[width=\linewidth]{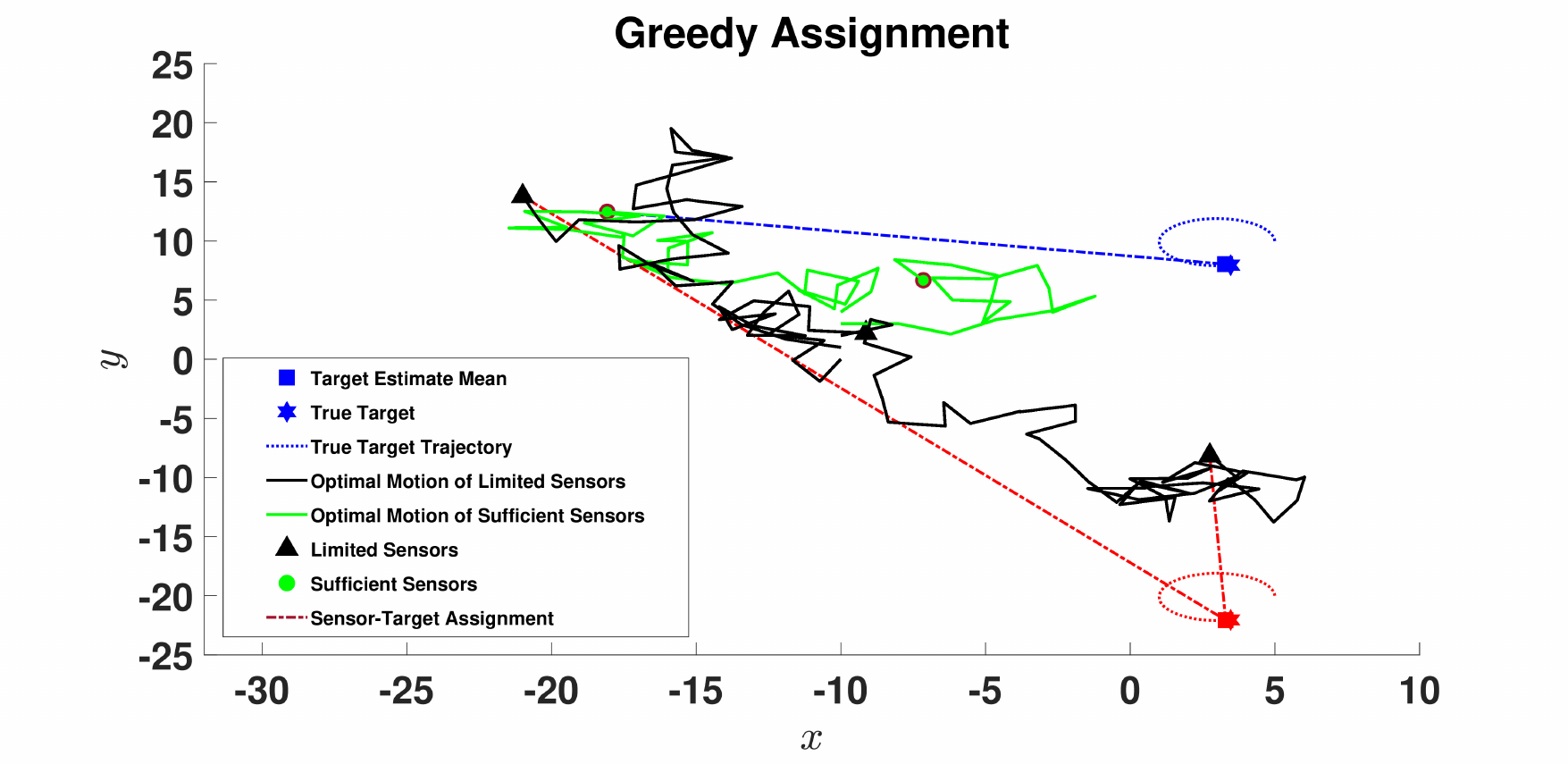}}}\\
\subfloat[\( k=100 \)]
{{\includegraphics[width=\linewidth]{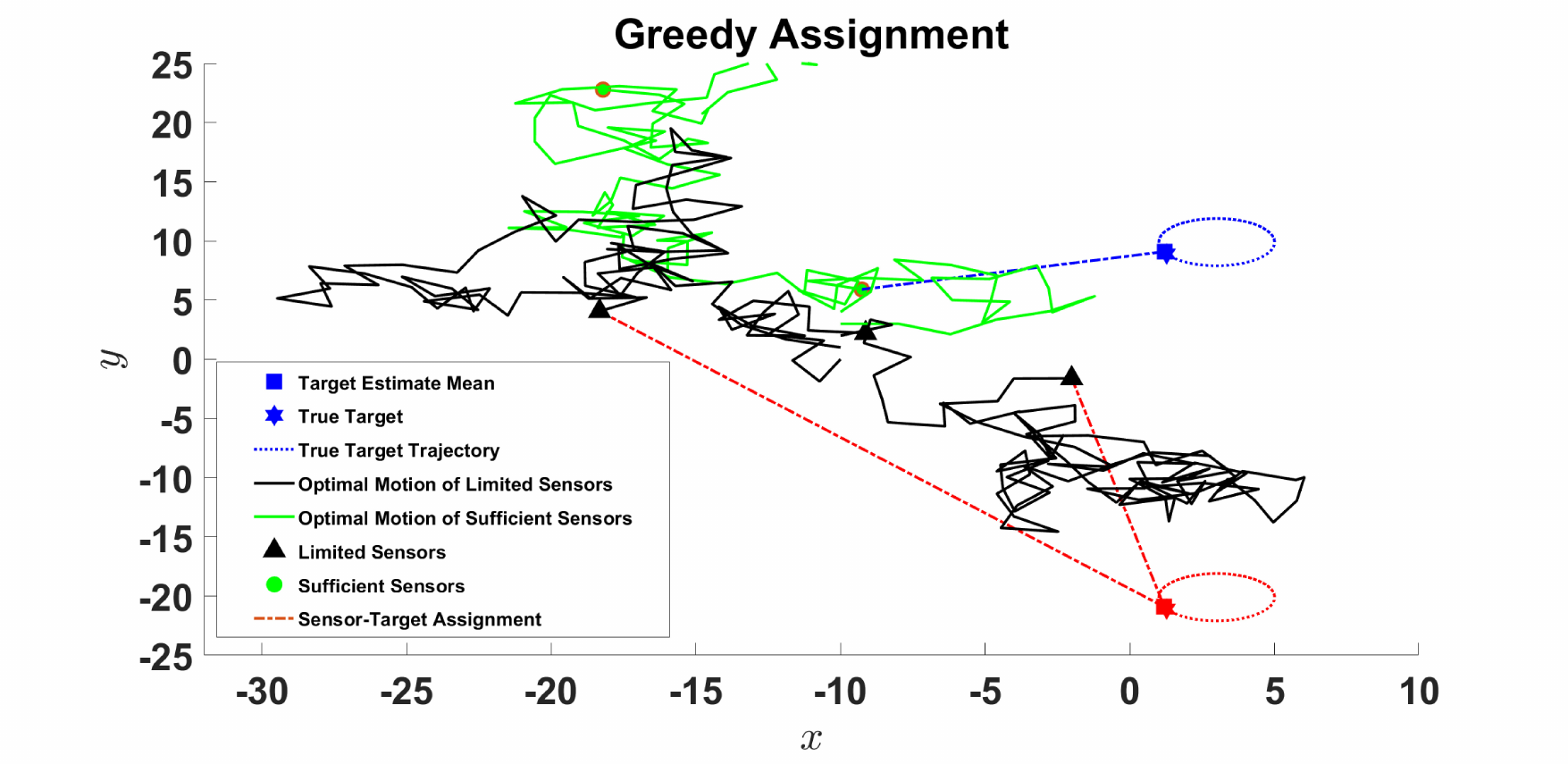}}}
\caption{Greedy assignment for multi-target tracking with circular motion. Symbols: Hexagram (true position), filled square (estimate mean), dotted circle (trajectory). Upward triangle (limited sensing robots), green circle (sufficient  robots). Dash-dotted lines (sensor-target assignment). Green solid lines (sufficient  robots' optimal movement), black solid lines (limited sensing robots' optimal movement).}\label{fig2}
\end{figure}
\subsection{Optimal Action}
In our simulations, in addition to the sensor assignment problem, we also consider the optimal selection of sensor actions. To be more specific, consider the Hamiltonian:
\begin{align}
&H(\mathbf{x}_{i,t}, \mathbf{y}_{j,t}, \mathbf{u}_{i,t}, \boldsymbol{\lambda}_{i,t}) = \boldsymbol{\lambda}^{1}_{i,t} (v_{i} \cos\theta_{i,t}) + \boldsymbol{\lambda}^{2}_{i,t}(v_{i} \sin\theta_{i,t}) \cr
&+ \boldsymbol{\lambda}^{3}_{i,t} \omega_{i} + \frac{1}{2}\left[ (x_{i,t}^{1}-y_{j,t}^{1})^2 + (x_{i,t}^{2}-y_{j,t}^{2})^2 + (v_{i}^2+\omega_{i}^2)\right],
\end{align}
From Pontryagin's Maximum Principle \cite[Chap. 5]{n5,Kirk}, the optimal action satisfies:
\begin{align}
\begin{cases}
v_{i} + \boldsymbol{\lambda}^{1}_{i,t} \cos\theta_{i,t} + \boldsymbol{\lambda}^{2}_{i,t} \sin\theta_{i,t} = 0, \\
\omega_{i} + \boldsymbol{\lambda}^{3}_{i,t} = 0,
\end{cases}
\end{align}
with boundary conditions:
\begin{align}
\boldsymbol{\lambda}_{i,T} = 0, \quad \forall i \in \psi(j).
\end{align}
State variables have known initial conditions. In our evaluation, we have set \( M \in \{1, 2, 3\} \), \( N_1 = 2 \), \( N_2 = 3 \). Targets follow circular motion with \( v_j = 2 \, \text{m/s} \), \( \phi_j = 0.1 \, \text{rad/s} \). Robots have unicycle motion with \( v_{\text{max}} = 2 \, \text{m/s} \).

\begin{table*}
    \centering
      \caption{Root Mean Square Error (RMSE) for two targets at different time steps using the greedy assignment algorithm \ref{al1} and the optimal assignment.}
    \label{table1}
    \begin{tabular}{|c|c|c|c|c|}
        \hline
            & \multicolumn{2}{|c|}{Algorithm \ref{al1}}&\multicolumn{2}{|c|}{Optimal Assignment} \\ \hline
        Time step & RMSE of Target 1 & RMSE of Target 2& RMSE of Target 1 & RMSE of Target 2 \\ \hline
   \( k=1 \)& \( 0.7071 \,\mathrm{m} \)& \( 0.7071 \,\mathrm{m}\) & \( 0.7071 \,\mathrm{m} \)& \( 0.7071 \,\mathrm{m}\)\\ \hline
   \( k=10 \)& \( 0.3518 \,\mathrm{m} \)& \( 0.3181 \,\mathrm{m}\) & \( 0.3088 \,\mathrm{m} \)& \( 0.3155 \,\mathrm{m}\)\\ \hline
   \( k=20 \)& \( 0.2969 \,\mathrm{m} \)& \( 0.3123 \,\mathrm{m}\) & \( 0.2288 \,\mathrm{m} \)& \( 0.2670 \,\mathrm{m}\)\\ \hline
   \( k=30 \)& \( 0.2696 \,\mathrm{m} \)& \( 0.2900 \,\mathrm{m}\) & \( 0.1974 \,\mathrm{m} \)& \( 0.2482 \,\mathrm{m}\)\\ \hline
   \( k=40 \)& \( 0.2540 \,\mathrm{m} \)& \( 0.2824 \,\mathrm{m}\) & \( 0.1974 \,\mathrm{m} \)& \( 0.2362 \,\mathrm{m}\)\\ \hline
   \( k=50 \)& \( 0.2431 \,\mathrm{m} \)& \( 0.2635 \,\mathrm{m}\) & \( 0.1844 \,\mathrm{m} \)& \( 0.2309 \,\mathrm{m}\)\\ \hline
   \( k=75 \)& \( 0.2322 \,\mathrm{m} \)& \( 0.2896 \,\mathrm{m}\) & \( 0.1819 \,\mathrm{m} \)& \( 0.2255 \,\mathrm{m}\)\\ \hline
   \( k=100 \)& \( 0.2318 \,\mathrm{m} \)& \( 0.2846 \,\mathrm{m}\) & \( 0.1725 \,\mathrm{m} \)& \( 0.2208 \,\mathrm{m}\)\\ \hline
    \end{tabular}
  \end{table*}
Figure~\ref{fig2} illustrates the greedy assignment strategy for multi-target tracking using sufficient  (green circles) and limited (triangles) sensing robots at \( k = 50 \) and \( k = 100 \). Targets move along a circular path (hexagram: true position; filled square: estimate; dotted circle: trajectory), with dashed lines indicating sensor assignments. Robots adjust their positions over time to maintain tracking accuracy, demonstrating the strategy’s effectiveness.

Figure~\ref{fig3} plots the Euclidean error for two targets across 100 steps. Target 1 (red) maintains an error below \(0.25\,\text{m}\) for most steps, while Target 2 (blue) shows a higher initial error, likely due to initialization, but later converges below \(0.4\,\text{m}\). Table~\ref{table1} summarizes the Root Mean Square Error (RMSE) for both targets at selected time steps, comparing the greedy assignment algorithm (Algorithm \ref{al1}) with the optimal assignment strategy. For Target 1, the greedy algorithm’s RMSE decreases from 0.7071 m at \( k=1 \) to \(0.2318 \,\mathrm{m}\) at \( k=100 \), while for Target 2, it decreases from \(0.7071\,\mathrm{m}\) to \(0.2635\,\mathrm{m}\) at \( k=50 \) but slightly increases to \(0.2846\,\mathrm{m}\) at \( k=100 \), possibly due to sensor limitations. The optimal assignment consistently achieves lower RMSEs, reaching \(0.1725\,\mathrm{m}\) for Target 1 and \(0.2208\,\mathrm{m}\) for Target 2 at \( k=100 \). Despite this, the greedy algorithm remains competitive, with RMSEs close to optimal, suggesting its robustness for real-time applications.

\begin{figure}
\centering
{{\includegraphics[width=9cm, height=4.5cm]{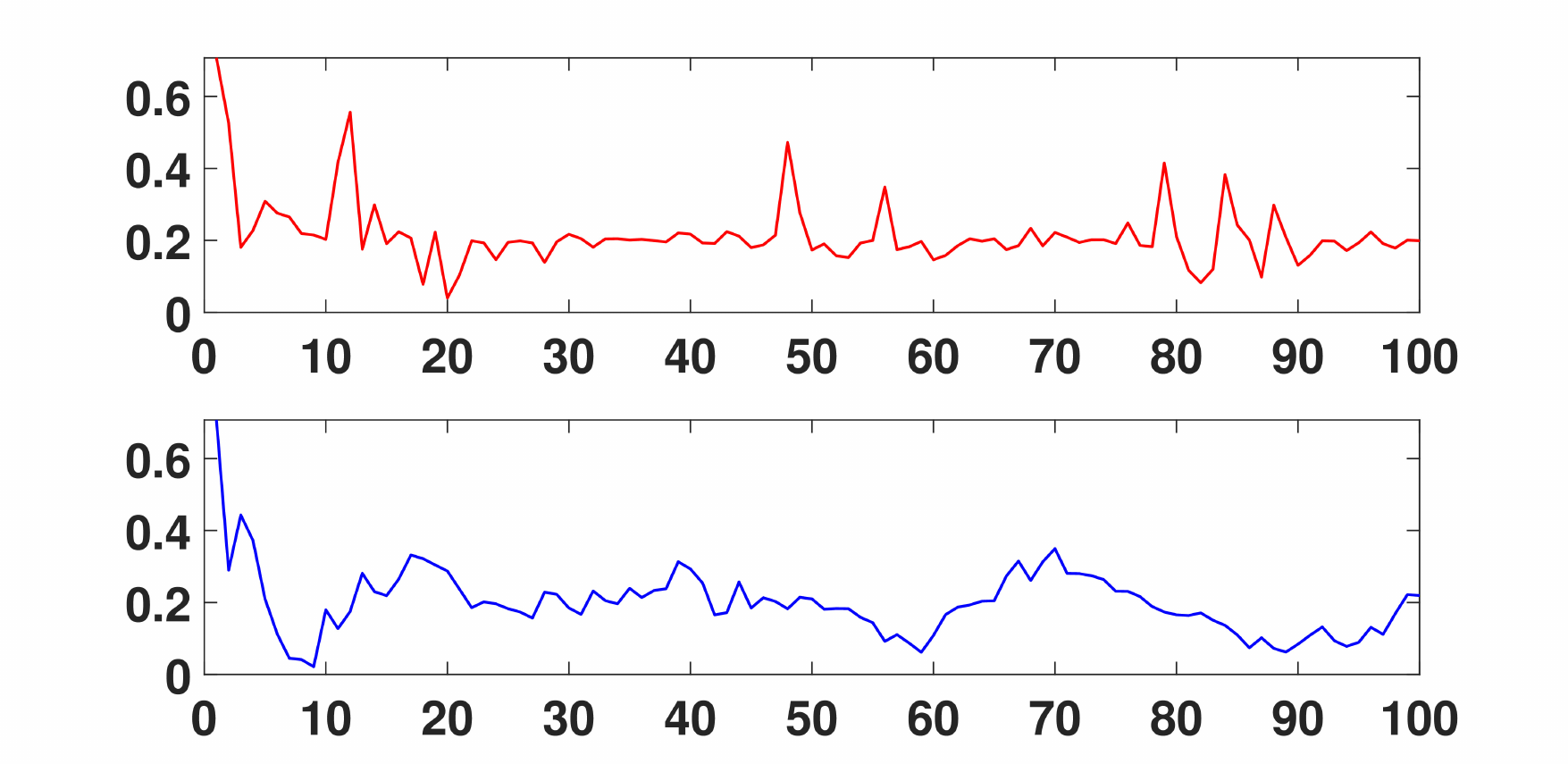}}}
\caption{Euclidean error related to two targets within 100 time steps using the greedy assignment}\label{fig3}
\end{figure}

\begin{figure}
\centering
{{\includegraphics[width=9cm, height=4.5cm]{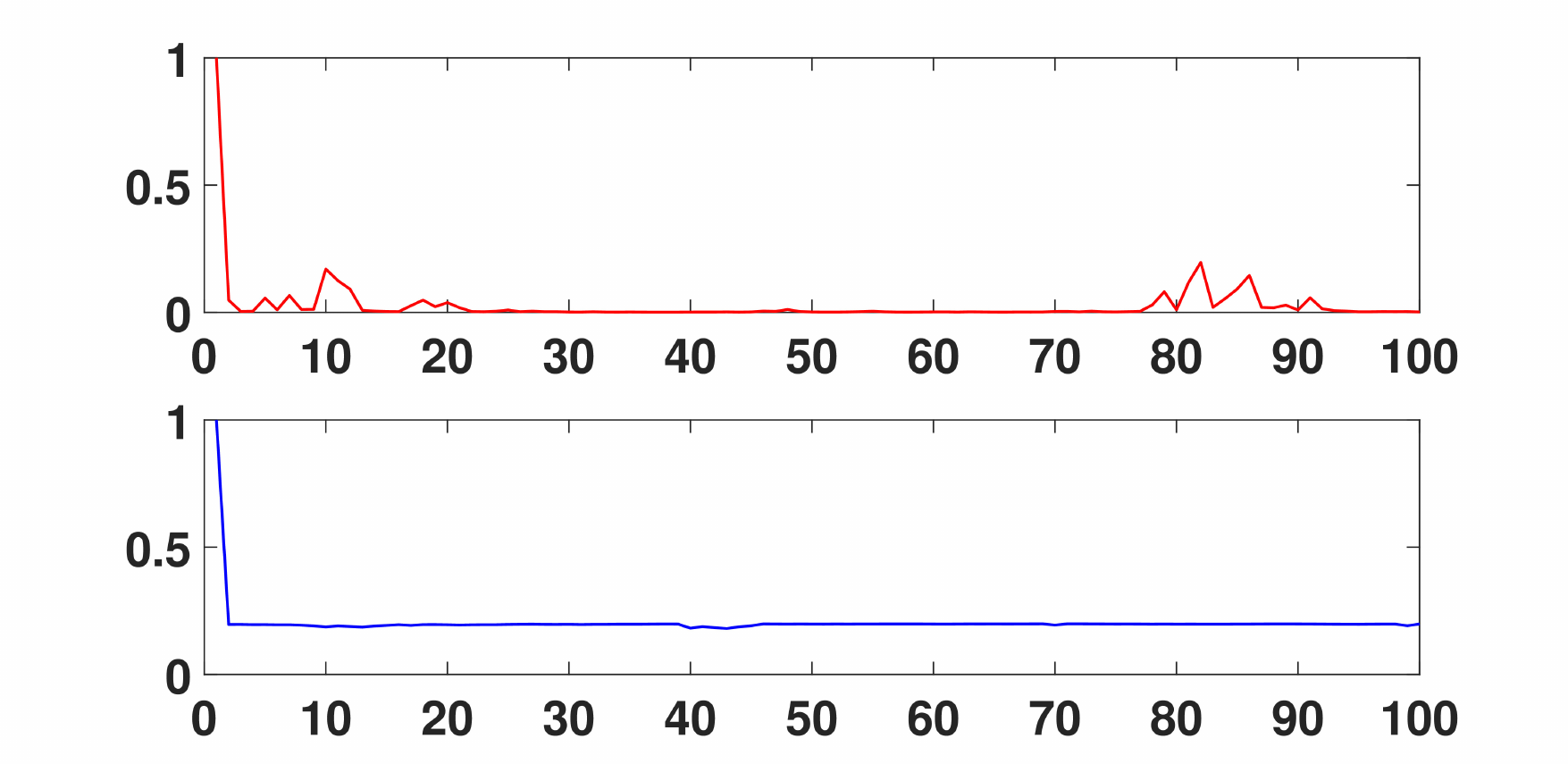}}}
\caption{Comparison of trace of estimated covariance for two targets within 100 time steps by the greedy assignment.}\label{fig4}
\end{figure}

Figure~\ref{fig4} presents the trace of estimated covariance for two targets over 100 steps using the greedy assignment. The red curve exhibits a sharp initial decline followed by some minor fluctuations, while the blue curve stabilizes at a lower level (compared to Target 1), reflecting slightly more reliable tracking for Target 2. Figure~\ref{fig6} shows the total tracking quality for both the greedy and optimal assignments. Robot positions (limited and sufficient ) are randomly placed within a \(20\,\,\mathrm{m} \times 20\,\,\mathrm{m}\) area centered at the origin, ensuring full target coverage. Targets are also randomly distributed in this domain.

The results confirm that greedy tracking quality remains between the optimal value and its half, consistent with the bound \( q(\text{GREEDY}) \geq \frac{1}{2} q(\text{OPT}) \). As target count increases, the greedy method maintains this performance, highlighting its efficiency and scalability for large-scale scenarios.

\begin{figure}[h!]
\centering
{{\includegraphics[width=9cm, height=5cm]{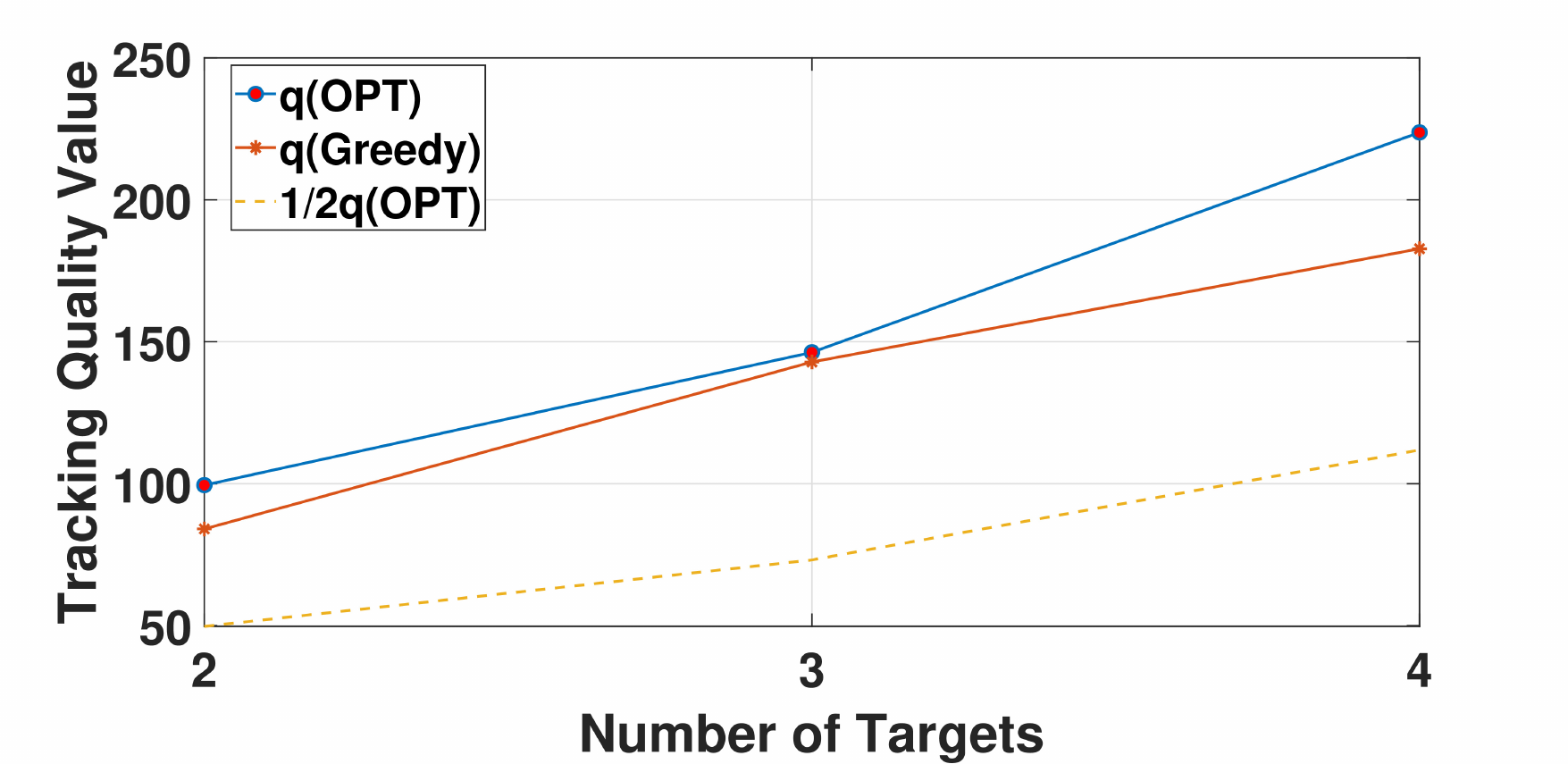}}}
\caption{Comparison of tracking quality values for optimal assignment, greedy performance, and half of the optimal assignment across varying target counts.}\label{fig6}
\end{figure}

\section{CONCLUSIONS}
This paper addresses the multi-robot, multi-target tracking problem with robots of varying sensing capabilities. By leveraging matroid theory, we propose efficient greedy algorithms that dynamically assign robots to targets to maximize tracking quality. The algorithms guarantee constant-factor approximation bounds, \(1/3\) for general tracking functions and \(1/2\) for submodular functions, while maintaining polynomial-time complexity. Simulations confirm their effectiveness in accurate, long-term tracking, demonstrating robustness and scalability for real-world applications such as surveillance and monitoring. Future work will extend the framework to dynamic environments with varying numbers of sensors and targets, further strengthening performance guarantees and robustness in practice.


\end{document}